\documentclass[conference]{IEEEtran}
\IEEEoverridecommandlockouts
\usepackage{cite}
\usepackage{amsthm}
\usepackage{amssymb}
\usepackage{textcomp}
\usepackage{xcolor}
\usepackage{array}
\usepackage{caption}
\usepackage{graphicx}
\usepackage{amsmath}
\usepackage{algorithm}
\usepackage{float}
\usepackage{hyperref}
\usepackage[noend]{algpseudocode}
\def\BibTeX{{\rm B\kern-.05em{\sc i\kern-.025em b}\kern-.08em
    T\kern-.1667em\lower.7ex\hbox{E}\kern-.125emX}}

\newtheorem{theorem}{Theorem}
\newtheorem{lemma}{Lemma}
\theoremstyle{definition}

\begin{document}

\title{Hyperedge Modeling in Hypergraph Neural Networks by using Densest Overlapping Subgraphs}

\author{\IEEEauthorblockN{Mehrad Soltani}
\IEEEauthorblockA{\textit{University of Windsor} \\
Windsor, Canada \\
soltani8@uwindsor.ca}
\and
\IEEEauthorblockN{Luis Rueda}
\IEEEauthorblockA{\textit{University of Windsor} \\
Windsor, Canada \\
lrueda@uwindsor.ca}
}

\maketitle

\begin{abstract}
Hypergraphs tackle the limitations of traditional graphs by introducing {\em hyperedges}. While graph edges connect only two nodes, hyperedges connect an arbitrary number of nodes along their edges. Also, the underlying message-passing mechanisms in Hypergraph Neural Networks (HGNNs) are in the form of vertex-hyperedge-vertex, which let HGNNs capture and utilize richer and more complex structural information than traditional Graph Neural Networks (GNNs).
More recently, the idea of overlapping subgraphs has emerged. These subgraphs can capture more information about subgroups of vertices without limiting one vertex belonging to just one group, allowing vertices to belong to multiple groups or subgraphs. In addition, one of the most important problems in graph clustering is to find densest overlapping subgraphs (DOS).
In this paper, we propose a solution to the DOS problem via Agglomerative Greedy Enumeration (DOSAGE) algorithm as a novel approach to enhance the process of generating the densest overlapping subgraphs and, hence, a robust construction of the hypergraphs. Experiments on standard benchmarks show that the DOSAGE algorithm significantly outperforms the HGNNs and six other methods on the node classification task. 
\end{abstract}


\section{Introduction}
While Graph Neural Networks (GNNs) have attracted increasing attention in the past few years, they suffer from the limitation in the assumption of pairwise connections between nodes, which cannot capture the complex relationships between neighbor nodes and limits the capability of high-order correlation modeling. Even message-passing mechanisms enable GNNs to capture even beyond pairwise connections among vertices, they still suffer from the ability to capture indirect relationships \cite{Consistency}.
In this regard, Hypergraph Neural Networks (HGNN) have been proposed to address the challenges of representation learning using high-order correlations \cite{HGNNs}. HGNNs represent data in the hypergraph structure \cite{Learning-with-hypergraphs}. Graphs become hypergraphs when additional information in a graph groups entity nodes together into sets \cite{Hypergraph-visualization}. Also, the message passing in the HGNNs is a two-step process. In this first step, information from vertices is aggregated into the hyperedges to which they belong. In the second step, the aggregated information in each hyperedge is then propagated back to the vertices. Each vertex updates its representation by aggregating information from all its hyperedges. This step effectively allows the sharing of information among all vertices that are connected both directly and indirectly via common hyperedges \cite{HGNNs}.

There is no explicit hypergraph structure in most cases \cite{Methods}. As such, it is necessary to generate a good hypergraph structure to make the most of the high-order correlation among the data. Generally speaking, a hypergraph is created based on two different data structures when the data correlation is without graph structure and when the data correlation is with graph structure. Since most of the data that we have available are in the form of a graph or can be converted into a graph, we target this area for our work. There have been many attempts to identify important subsets of vertices within a graph, and these subsets can be identified in various forms that can either be overlapping \cite{Top-k} subgraphs, unlike cuts in graphs that partition the vertices of a graph into two disjoint subsets. 

In this paper, we introduce a novel approach to HGNNs, which considers the densest overlapping subgraphs \cite{Top-k} in the hypergraph modeling step. To identify the top-$K$ densest subgraphs, we consider a constrained version of the problem, which we call constrained top-$k$-overlapping densest subgraphs (CTODS). This problem, which is shown to be NP-complete, is solved via a new algorithm, which we call densest overlapping subgraphs via agglomerative greedy enumeration (DOSAGE). Our method is not only able to identify the high-correlated subgraphs but also uses the objective function, which takes into account both the density of the subgraphs and the distance between subgraphs in a constrained form that limits the size of each subgraph, as well as the density while ensuring full coverage of the entire graph.

The paper is organized as follows. First, we discuss existing methods for hypergraph modeling and where they fall short. Secondly, we will discuss the top-$K$ densest subgraphs and how, by refining them, we created our DOSAGE algorithm. In the experiment section, we discuss the result of our method in comparison with other hypergraph modeling methods and GNNs. Finally, conclusions and future works are discussed in the last section. The main contributions of this paper are summarized as follows:
(i) enhance HGNN accuracy with a new hypergraph modeling method based on the densest overlapping subgraphs;  
(ii) design a new algorithm for finding the densest overlapping subgraphs; 
(iii) define a new problem: the constrained overlapping subgraphs with full coverage, specific subgraph size, and diameter.

\section{Related Work}
Hypergraphs were first introduced in \cite{Learning-with-hypergraphs} by defining hypergraph as a generalization of a graph in which edges, known as hyperedges, can connect any number of vertices, not just two. This allows for more complex relationships among the objects of interest than simple graphs. 

The authors of \cite{Evolution-of-cooperation} used hypergraph to distinguish the dynamics within an $m$-clique, where each of the $m$ individuals interacts separately in pairs with each of the other $m-1$ individuals, from the dynamics where the $m$ individuals interact all together as a group. It does not provide the necessary isometric properties required for faithful hypergraph representation. This transformation does not capture the full complexity of hypergraph structures, leading to a loss of critical relational information.

Methods that generalize graph edit distance to hypergraphs, such as those proposed in \cite{Hypergraph-co-optimal}, face significant computational challenges. The graph edit distance problem has been found to be NP-hard to compute and APX-hard to approximate, making it impractical for large-scale hypergraph applications. Also, other distance based methods performance that have been proposed in \cite{Methods}, rely heavily on the accuracy of the distance measurement between vertices.

In \cite{Methods}, representation-based methods have been proposed, which construct hyperedges by using feature reconstruction techniques. This method should solve optimization problems to determine reconstruction coefficients, which is computationally expensive, especially when we deal with large datasets.

Some frameworks, such as the one described in \cite{A-tensor-based-algorithm}, construct hypergraphs with a fixed size for each hyperedge. This approach is unsuitable for scenarios where the input data consists of arbitrary hypergraphs with varying hyperedge sizes. 

\begin{table}[h!]
\centering
\caption{Notation and definitions.}
\renewcommand{\arraystretch}{1.5}
\begin{tabular}{|p{2.5cm}|p{5.5cm}|}
\hline
\textbf{Notation} & \textbf{Definition} \\
\hline
\(\mathcal{G}_h^{(\alpha, \beta)} = (\mathcal{V},\mathcal{E})\) & A hypergraph with vertex set \(\mathcal{V}\) and hyperedge set \(\mathcal{E}\), where each hyperedge \(e \in \mathcal{E}\) satisfies \(\alpha \leq |e| \leq \beta\). \\ 
\hline
\(\mathcal{G}_h = (\mathcal{V}, \mathcal{E}, \mathbf{W})\) & Indicates a weighted hypergraph, where \(\mathcal{V}\) is the set of vertices, \(\mathcal{E}\) is the set of hyperedges, and \(\mathbf{W}\) represents the weights of the hyperedges. \\
\hline
\( \mathcal{G}_h = (\mathcal{V}, \mathcal{E}) \) & A hypergraph with the vertex set \( \mathcal{V} \) and the hyperedge set \( \mathcal{E} \). \\
\hline
\( \mathcal{G} = (\mathcal{V}, \mathcal{E}') \) & A simple graph with the vertex \( \mathcal{V} \) and the edge set \(\mathcal{E}'\) . \\
\hline
\( \mathcal{E}'(S) \) & The set of edges in the induced subgraph \( \mathcal{G}[S] \), where \( \mathcal{G}[S] \) is the subgraph of \( \mathcal{G} \) induced by the vertices in \( S \). \\
\hline
\(\mathcal{V}\) & The set of vertices in the graph \(\mathcal{G}\) that are in one hyperedge. \\
\hline
\(\mathcal{E}\) & The set of hyperedges in the hypergraph \(\mathcal{G}_h\). \\
\hline
\(N\) & The number of vertices in \(\mathcal{G}_h\), i.e., \(|\mathcal{V}|\). \\
\hline
\(M\) & The number of hyperedges in \(\mathcal{G}_h\), i.e., \(|\mathcal{E}|\). \\
\hline
\(S\) & A subset of vertices in the hypergraph \(\mathcal{G}_h\). \\
\hline
\( S^c \) & For a vertex subset \( S \subset \mathcal{V} \), denote the complement of \( S \). \\
\hline
\(\mathbf{H}\) & The incidence matrix of the hypergraph \(\mathcal{G}_h\). \\
\hline
\(\mathbf{H}'\) & The incidence matrix of the graph \(\mathcal{G}\).\\
\hline
\(x_i^0\) & The initial feature for the \(i\)-th vertex in \(\mathcal{G}_h\). \\
\hline
\(\mathbf{X}_0\) & The initial features for all vertices in \(\mathcal{G}_h\). \\
\hline
\(\mathbf{X}^t\) & The input feature of the convolution layer \(t\). \\
\hline
\(\mathbf{X}_i^t\) & The embedding for vertex \(i\) in layer \(t\). \\
\hline
\(\mathbf{W}\) & The diagonal matrix of the hyperedge weights. \\
\hline
\(d(v)\) & The degree of vertex \(v\). \\
\hline
\(\delta(e)\) & The degree of hyperedge \(e\). \\
\hline
\(\mathbf{D}_v\) & The diagonal matrix of vertex degrees. \(\mathbf{D}_v \in \mathbb{R}^{N \times N}\). \\
\hline
\(\mathbf{D}_e\) & The diagonal matrix of hyperedge degrees. \(\mathbf{D}_e \in \mathbb{R}^{M \times M}\). \hspace{1cm} \\
\hline
\( w(e) \) & A positive number associated with each hyperedge. \\
\hline

\end{tabular}
\label{table:notations}
\end{table}

\section{Preliminaries of Hypergraphs}
Before diving into what a hypergraph is and how we can generate hypergraphs, let us first discuss and define graphs and hypergraphs.
\subsection{Graphs and Hypergraphs}
First, we review the basic concepts of graphs and hypergraphs. Then, we summarize this paper's important notations and definitions in Table \ref{table:notations}. 

Let \( \mathcal{V} \) be a (typically finite) set of elements, nodes, or objects, which we formally call "vertices", and \( \mathcal{E}' \) be a set of pairs of vertices. Given that, then for two vertices \( u, v \in \mathcal{V} \), an edge is a set \( \{u, v\} \in \mathcal{E}' \), indicating that there is a connection between \( u \) and \( v \). It is then common to represent \( \mathcal{E}' \) as either a boolean adjacency matrix \( \mathcal{A} \) where \( \mathcal{A} \in \{0,1\}^{|\mathcal{V}| \times |\mathcal{V}|} \), where an entry \( \mathcal{A}_{ij} \) is 1 if \( v_i \) and \( v_j \) are connected in \( \mathcal{E}' \); or as an incidence matrix \(\mathbf{H}'\), where now also \( \mathbf{H}' \in \{0,1\}^{|\mathcal{V}| \times |\mathcal{E}'|} \), and an entry \( \mathbf{H}'_{ij} \) is now 1 if the vertex \( v_i \) is in edge \( e'_j \).

Let \( \mathcal{V} \) denote a finite set of elements, nodes, or objects, which we formally call "vertices". Let \( \mathcal{E} \) be a family of subsets \( e \) of \( \mathcal{V} \) such that \( \bigcup_{e \in \mathcal{E}} e = \mathcal{V} \). Then we call \( \mathcal{G}_h = (\mathcal{V}, \mathcal{E}) \) a hypergraph with the vertex set \( \mathcal{V} \) and the hyperedge set \( \mathcal{E} \). A hyperedge containing just two vertices is a simple graph edge. A weighted hypergraph is a hypergraph that has a positive number \( w(e) \) associated with each hyperedge \( e \), showing the importance of the connections inside a hyperedge, called the weight of hyperedge \( e \). Denote a weighted hypergraph by \( \mathcal{G}_h = (\mathcal{V}, \mathcal{E}, \mathbf{W})\) in which \( \mathbf{W} \) denote the diagonal matrix containing the weights of hyperedges. Furthermore, we introduce a hypergraph with constraints on the hyperedge sizes as follows: 
\[
\mathcal{G}_h^{(\alpha, \beta)} = (\mathcal{V}, \mathcal{E})
\]
where each hyperedge \(e \in \mathcal{E} \) satisfies the constraint \(\alpha \leq |e| \leq \beta\), where \(|e|\) denotes the number of vertices in hyperedge \(e\). Later in this paper, we will show why we need constraints on the hyperedge sizes, but for now, let us discuss the general hypergraph.

% \begin{figure}[t]
% \includegraphics[width=0.5\textwidth]{figure2.png}
% \captionsetup{width=0.45\textwidth}
% \caption{\cite{Learning-with-hypergraphs} Hypergraph vs. simple graph. Left: An author set $\mathcal{E} = \{e_1, e_2, e_3\}$ (representing hyperedges) and an article set $\mathcal{V} = \{v_1, v_2, v_3, v_4, v_5, v_6, v_7\}$ (representing vertices). The incidence matrix $\mathbf{H}$ is defined such that the entry $(v_i, e_j)$ is set to 1 if $e_j$ (an author) is associated with article $v_i$ and 0 otherwise. Middle: A simple graph $\mathcal{G} = (\mathcal{V}, \mathcal{E}')$ is depicted where an edge in $\mathcal{E}'$ joins two articles if they share at least one common author. 
% % This simple graph cannot capture the full complexity of relationships, such as whether the same author has written three or more articles. 
% Right: A hypergraph $\mathcal{G}_h = (\mathcal{V}, \mathcal{E})$ that completely illustrates the complex relationships among authors and articles by allowing each hyperedge to connect multiple vertices, representing an author contributing to multiple articles.}
% \label{fig}
% \end{figure}

\cite{Learning-with-hypergraphs} A hypergraph \( \mathcal{G}_h \) can be represented by a \( |\mathcal{V}| \times |\mathcal{E}| \) matrix \(\mathbf{H}\) with entries \( h(v, e) = 1 \) if \( v \in e \) and 0 otherwise. This is called the incidence matrix of \( \mathcal{G}_h \). Then for a vertex \( v \in \mathcal{V} \), its vertex degree is defined as 
\[ d(v) = \sum_{e \in \mathcal{E}} w(e) H(v, e). \] 
For a hyperedge \( e \in \mathcal{E} \), its edge degree is defined as \[ d(e) = \sum_{v \in \mathcal{V}} H(v, e). \] \( D_v \) and \( D_e \) denote the diagonal matrices of vertex degrees and edge degrees, respectively. The initial feature set for each vertex is denoted as  \( X_0 = \{ x_0^1, x_0^2, \ldots, x_0^N \}\) for\(\quad x_0^i \in \mathbb{R}^{C_0},\) where \( C_0 \) is the dimension of the feature.

The adjacency matrix \( \mathcal{A} \) of hypergraph \( \mathcal{G}_h \) is defined as 
\[
\mathcal{A} = \mathbf{H} \mathbf{W} \mathbf{H}^T - \mathbf{D}_v,
\]
where \( \mathbf{H}^T \) is the transpose of \(\mathbf{H}\).

For a vertex subset \( S \subset \mathcal{V} \), let \( S^c \) denote the complement of \( S \). A cut of a hypergraph \( \mathcal{G}_h = (\mathcal{V}, \mathcal{E}, \mathbf{W}) \) is a partition of \( \mathcal{V} \) into two parts \( S \) and \( S^c \). We say that a hyperedge \( e \) is cut if it is incident with the vertices in \( S \) and \( S^c \) simultaneously \cite{Learning-with-hypergraphs}.

Given a vertex subset \( S \subset \mathcal{V} \), define the hyperedge boundary \( \partial S \) of \( S \) to be a hyperedge set which consists of hyperedges which are cut, i.e.
\[
\partial S := \{e \in \mathcal{E} \mid e \cap S \neq \emptyset, e \cap S^c \neq \emptyset\},
\]
and \cite{tag} define the volume \( \text{vol}(S) \) of \( S \) to be the sum of the degrees of the vertices in \( S \), that is,
\[
\text{vol}(S) := \sum_{v \in S} d(v).
\]
Moreover, define the volume of \( \partial S \) by
\[
\text{vol}(\partial S) := \sum_{e \in \partial S} \frac{w(e) |e \cap S| |e \cap S^c|}{\delta(e)}.
\]
Clearly, we have \( \text{vol}(\partial S) = \text{vol}(\partial S^c) \).
The definition given by the above equation can be understood as follows. Let us imagine each hyperedge \( e \) as a clique \cite{Learning-with-hypergraphs}, i.e., a fully connected subgraph. To avoid unnecessary confusion, we call the edges in such an imaginary subgraph the subedges. Moreover, we assign the same weight \( \frac{w(e)}{\delta(e)} \) to all subedges. Then, when a hyperedge \( e \) is cut, there are \( |e \cap S| |e \cap S^c| \) subedges that are cut, and hence, a single sum term in the above equation is the sum of the weights over the subedges which are cut. Naturally, we try to obtain a partition in which the connection among the vertices in the same cluster is dense while the connection between two clusters is sparse. Using the above-introduced definitions, we may formalize this natural partition as
\[
\arg \min_{\emptyset \neq S \subset \mathcal{V}} c(S) := \frac{\text{vol}(\partial S)}{\left(\frac{1}{\text{vol}(S)} + \frac{1}{\text{vol}(S^c)}\right)}.
\]
% For a simple graph, \( |e \cap S| = |e \cap S^c| = 1 \), and \( \delta(e) = 2 \). Thus, the right-hand side of the above equation reduces to the simple graph normalized cut up to a factor of 1/2.

Since we use the densest overlapping subgraphs for hyperedge generation, let us discuss what are top-$K$ overlapping subgraphs and why we use them instead of conventional methods.

\subsection{Top-k-Overlapping Densest Subgraphs}
\cite{Top-k} Given a graph \(\mathcal{G} = (\mathcal{V}, \mathcal{E}')\), and a subset \(S \subseteq \mathcal{V}\), we denote by \(\mathcal{G}[S]\) the subgraph of \(\mathcal{G}\) induced by \(S\), formally \(\mathcal{G}[S] = (S, \mathcal{E}'(S))\), where \(\mathcal{E}'(S)\) is defined as follows:

\[
\mathcal{E}'(S) = \{ \{u, v\} : \{u, v\} \in \mathcal{E}' \ \text{and} \ u, v \in S \}.
\]

As we discussed in the literature review section, many approaches tried to use subgraphs as hyperedges. \cite{Learning-with-hypergraphs} define a hyperedge as a clique that is a fully connected subgraph. Although this definition is good for understanding what a hyperedge is, there are some problems with cliques as hyperedges. The main problem is that the impact of small variations in the topology of the graph has a huge impact on the cliques, which means they are sensitive to small changes. For instance, in the figure \ref{clique}, the number of cliques would break down in half by removing one edge from our subgraph. Because of that, we lose vital relationships among data when trying to form a hyperedge based on a clique. 
That is why we need a more powerful way of defining and creating hyperedges based on a subgraph.

\begin{figure}[t]
\includegraphics[width=0.5\textwidth]{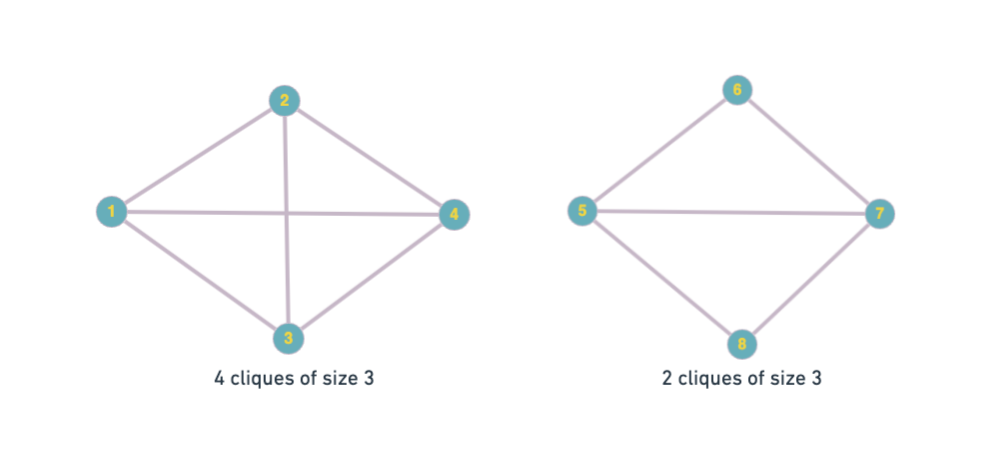}
\captionsetup{width=0.45\textwidth}
\caption{Impact of small variation in the topology of the graph on cliques}
\label{clique}
\end{figure}

The density of a subgraph based on measures other than cliques could be a good option. The Densest Subgraph problem aims at ﬁnding a single densest subgraph in a graph. Goldberg's algorithm \cite{goldberg} was the first one to propose a method for finding the densest subgraph, and other methods have been using this method as a groundwork for their algorithm. However, in many applications like hypergraph modeling, it is of interest ﬁnding a collection of dense subgraphs of a given graph. Also, dense subgraphs are related to non-disjoint communities in many real-world cases \cite{Top-k}. One example could be hubs, which are vertices that are part of several communities \cite{scale-free}. Hence, we need a method that, instead of giving the densest subgraph, finds a collection of subgraphs having maximum density in a given graph. We define the density of subgraphs in the graph as follows:

\[
\text{dens}(\mathcal{G}[S]) = \frac{|\mathcal{E}'(S)|}{|S|}
\]
Thus, the density is the ratio of the total number of edges to the total number of vertices in the subgraph.

Given the fact that some hyperedges might share some vertices too, the subgraphs we are trying to find to form our hyperedges must share some vertices too, and hence, we should do that by letting our subgraphs overlap with each other \cite{Top-k}. It is important to control the amount of overlapping between subgraphs since allowing overlaps leads to a solution that may contain $k$ copies of the same subgraph. To ensure distinctness between the subgraphs, we define a distance function \cite{Top-k} between two subgraphs as: 

\[
d(\mathcal{G}[U], \mathcal{G}[Z]) = 
\begin{cases} 
2 - \frac{|U \cap Z|^2}{|U||Z|} & \text{if } U \neq Z, \\
0 & \text{else}.
\end{cases}
\]

\noindent where \(\mathcal{G}[U]\) and \(\mathcal{G}[Z]\) denote the subgraphs induce by the vertex subsets \(U\) and \(Z\), respectively. Also, \({|U \cap Z|^2}\) is the number of vertices in the intersection of subsets \(U\) and \(Z\). It penalizes subgraphs with a high degree of overlap, which helps ensure that the selected subgraphs are sufficiently distinct from each other. The term {\small $2 - \frac{|U \cap Z|^2} {|U||Z|}$} makes sure that the distance is minimal when the overlap is maximal and vice versa, and it is encouraging subgraphs to be more distinct. 
% Also, this term is chosen to deal with the cases where subgraphs do not have a significant overlap.
When subgraphs share a large number of vertices, the overlap term {\small $\frac{|U \cap Z|^2} {|U||Z|}$} becomes significant, hence reducing the distance and contribution of such subgraphs to the objective function. The distance is bounded between 0 and 2, which makes the distance measure consistent.

Our Top-$K$ overlapping subgraphs algorithm looks for a collection of $K$ subgraphs that maximize an objective function that takes into account both the density of the subgraphs and the distance between the subgraphs of the solution, thus allowing overlap between the subgraphs, which depends on a parameter, \(\lambda\).

\[
r(W) = \text{dens}(W) + \lambda \sum_{i=1}^{k-1} \sum_{j=i+1}^{k} d(\mathcal{G}[W_i], \mathcal{G}[W_j])
\]
where \( W = \{\mathcal{G}[W_1], \mathcal{G}[W_2], \ldots, \mathcal{G}[W_k]\} \) is the set of top-k subgraphs, k is less than the number of vertices in the graph, \( \text{dens}(W) \) is the sum of the densities of the subgraphs in \( W \), and \( \lambda > 0 \) is a parameter that controls the trade-off between density and diversity of the subgraphs. When \(\lambda\) is small, then the density plays a dominant role in the objective function, so the output subgraphs can share a signiﬁcant part of vertices. On the other hand, if \(\lambda\) is large, then the subgraphs share few or no vertices, so that the subgraphs may be disjoint \cite{Top-k}.

This feature makes the top-$K$ overlapping subgraphs a great approach for our method. However, we discuss in the methodology section that finding the overlapping densest subgraphs problem is NP-complete. As such, we provide an efficient yet sub-optimal algorithm that reduces its complexity while empowering the hypergraph construction and, hence, the underlying machine learning tasks.

\section{Problem Definition}
We address the problem of hyperedge generation in hypergraphs by utilizing the concept of top-$K$ densest subgraphs \cite{Top-k}. The goal is to create hyperedges that capture indirect relationships in the data, which is not possible with simple pairwise connections in traditional graph structures. Also, we will show in the next section that finding the densest overlapping subgraphs is np-complete, and because of that, we define constrained overlapping subgraphs as a new problem and propose {\em DOSAGE} algorithm for finding them.

Given a simple graph \( \mathcal{G} = (\mathcal{V}, \mathcal{E}') \), where \( \mathcal{V} \) is the set of vertices and \( \mathcal{E}' \) is the set of edges, we aim to identify and utilize densest overlapping subgraphs to form hyperedges in a hypergraph \(\mathcal{G}_h^{(\alpha, \beta)} = (\mathcal{V}, \mathcal{E})\), where each hyperedge \(e \in \mathcal{E} \) satisfies the constraint \(\alpha \leq |e| \leq \beta\). \\
The problem can be formulated as follows:

Identify the top-k subsets \( S_1, S_2, \ldots, S_k \subseteq \mathcal{V} \) such that the density of each subset is maximized while also ensuring that the size of each subgraph \( S_i \) is between \(\alpha\) and \(\beta\). More formally, this can be expressed as follows:
\[
S_i = \arg\max_{\substack{S \subseteq \mathcal{V} \\ \alpha \leq |S| \leq \beta}} \text{dens}(\mathcal{G}[S]) = \arg\max_{\substack{S \subseteq \mathcal{V} \\ \alpha \leq |S| \leq \beta}} \frac{|\mathcal{E}'(S)|}{|S|}
\]
for \( i = 1, 2, \ldots, k \). \\
Moreover, the objective function also considers the distance between these subgraphs, ensuring that the overlap between the subgraphs is controlled by a parameter \(\lambda\). This means that the selected subgraphs \( S_1, S_2, \ldots, S_k \) should not only be dense but also satisfy the distance constraints. %defined as follows:%
% \[
% r(W) = \text{dens}(W) + \underline{\lambda \sum_{i=1}^{k-1} \sum_{j=i+1}^{k} d(\mathcal{G}[W_i], \mathcal{G}[W_j])}
% \]
% \[
% d(\mathcal{G}[U], \mathcal{G}[Z]) = 
% \begin{cases} 
% 2 - \frac{|U \cap Z|^2}{|U||Z|} & \text{if } U \neq Z, \\
% 0 & \text{else}.
% \end{cases}
% \]
% where \(W_i\) and \(W_j\) are subgraphs corresponding to the sets \(S_i\) and \(S_j\), and \(d(\mathcal{G}[W_i], \mathcal{G}[W_j])\) measures the distance between them. \\$$
% Next, use each of these top-k densest subgraphs to form a hyperedge in \( \mathcal{G}_h \). Specifically, each hyperedge \( e_i \in \mathcal{E} \) corresponds to the vertex set \( S_i \):
% \[
% e_i = S_i \quad \text{for } i = 1, 2, \ldots, k.
% \] \\
Additionally, if a vertex belongs to a subgraph, it should also belong to the corresponding hyperedge. This ensures that the vertices included in the densest subgraphs are accurately represented in the hypergraph, maintaining the integrity of the high-order correlations:

\[
v \in S_i \implies v \in e_i \quad \text{for all } v \in \mathcal{V}, \quad i = 1, 2, \ldots, k.
\]

\section{Methodology}
Our {\em DOSAGE} algorithm creates hyperedges based on the overlapping densest subgraphs in a graph based on three parameters that are $K$, which is the total number of subgraphs and hence a number of hyperedges, the minimum and maximum size that is the minimum and maximum number of vertices in one hyperedge. Full coverage of the graph should also be taken into account since we want all the vertices to be represented by the hypergraph and also to make sure that we not only obtain the densest regions but also fully cover the whole graph so that we do not miss any important piece of information during the construction process.

\subsection{Complexity of the DOSAGE Algorithm}
In this part, we prove the NP-completeness of the constrained Top-k-Overlapping Densest Subgraphs problem by demonstrating a polynomial-time bidirectional reduction to and from known NP-complete problems. Specifically, we show that the problem can be reduced to and from the k-Clique or 3-Clique problems.
\begin{lemma}
Given an instance \(\mathcal{G} = (\mathcal{V}, \mathcal{E}')\) of the 3-Clique Partition problem, we can construct an instance of the constrained Top-k-Overlapping Densest Subgraphs problem such that if \(\mathcal{G}\) can be partitioned into three cliques, we can compute a set \(W = \{\mathcal{G}[S_1], \mathcal{G}[S_2], \mathcal{G}[S_3]\}\) where \(r(W) \geq \frac{|\mathcal{V}| - 3}{2} + 18|\mathcal{V}|^3\).
\end{lemma}

\begin{proof}
    We take the input graph \(\mathcal{G} = (\mathcal{V}, \mathcal{E}')\) from the 3-Clique Partition problem and use it as the input graph for the constrained Top-k-Overlapping Densest Subgraphs problem.
    If \(\mathcal{G}\) can be partitioned into three cliques \(\mathcal{V}_1\), \(\mathcal{V}_2\), and \(\mathcal{V}_3\), then the corresponding subgraphs \(\mathcal{G}[\mathcal{V}_1]\), \(\mathcal{G}[\mathcal{V}_2]\), and \(\mathcal{G}[\mathcal{V}_3]\) are used to construct the solution \(W = \{\mathcal{G}[S_1], \mathcal{G}[S_2], \mathcal{G}[S_3]\}\).
    The value \(r(W)\) is calculated based on the densities of these subgraphs and the distances between them. The result is that \(r(W) \geq \frac{|\mathcal{V}| - 3}{2} + 18|\mathcal{V}|^3\).
\end{proof}

\begin{lemma}
Given a solution \(W = \{\mathcal{G}[S_1], \mathcal{G}[S_2], \mathcal{G}[S_3]\}\) for the constrained Top-k-Overlapping Densest Subgraphs problem with \(r(W) \geq \frac{|\mathcal{V}| - 3}{2} + 18|\mathcal{V}|^3\), we can find a partition of \(\mathcal{G} = (\mathcal{V}, \mathcal{E}')\) into three cliques.
\end{lemma}

\begin{proof}
    We take the solution \(W\) from the constrained Top-k-Overlapping Densest Subgraphs problem and use the subgraphs \(\mathcal{G}[S_1]\), \(\mathcal{G}[S_2]\), and \(\mathcal{G}[S_3]\) to construct three cliques \(\mathcal{V}_1\), \(\mathcal{V}_2\), and \(\mathcal{V}_3\) in the original graph \(\mathcal{G} = (\mathcal{V}, \mathcal{E}')\).
    These cliques \(\mathcal{V}_1\), \(\mathcal{V}_2\), and \(\mathcal{V}_3\) must form a valid partition of \(\mathcal{V}\).
\end{proof}

\begin{theorem}
The constrained Top-$k$-Overlapping Densest Subgraphs problem is NP-complete.
\end{theorem}

\begin{proof}
We first formally define the decision problems, showing that the problem is in NP, and then provide polynomial-time reductions in both directions.

\textit{Constrained Top-k-Overlapping Densest Subgraphs (CTODS):}\\
\textbf{Given:} A graph \(\mathcal{G}_h = (\mathcal{V}, \mathcal{E})\), positive integers \(k\), \(\alpha\), \(\beta\), and a real number \(r\).\\
\textbf{Question:} Do there exist \(k\) subgraphs \(S_1, S_2, \ldots, S_k \subseteq \mathcal{V}\) such that:
\begin{enumerate}[(a)]
    \item \(\alpha \leq |S_i| \leq \beta\) for all \(i = 1, \ldots, k\),
    \item The density of each \(S_i\) is at least \(r\),
    \item The subgraphs satisfy the overlap constraint as defined by the distance function \(d(\mathcal{G}[U], \mathcal{G}[Z])\).
\end{enumerate}

\textit{3-Clique:}\\
\textbf{Given:} A graph \(\mathcal{G} = (\mathcal{V}, \mathcal{E}')\) and a positive integer \(k\).\\
\textbf{Question:} Does \(\mathcal{G}\) contain a clique of size 3?

\textbf{2. CTODS is in NP:}\\
To show CTODS is in NP, we prove that a solution can be verified in polynomial time. Given a set of \(k\) subgraphs, we can:
\begin{enumerate}[(a)]
    \item Verify the size constraints in \(O(k)\) time,
    \item Calculate the density of each subgraph in \(O(|\mathcal{V}| + |\mathcal{E}|)\) time,
    \item Verify the overlap constraints in \(O(k^2 \cdot |\mathcal{V}|)\) time.
\end{enumerate}
Thus, the verification can be done in polynomial time, so CTODS is in NP.

\textbf{3. Reduction from CTODS to 3-Clique:}\\
We construct a graph \(\mathcal{G}_P = (\mathcal{V}_P, \mathcal{E}_P)\) as follows:
\begin{itemize}
    \item Each vertex \(v \in \mathcal{V}_P\) represents a potential subgraph \(S_i\) in \(\mathcal{G}_h\) that satisfies the size and density constraints.
    \item Add an edge between vertices \(u, v \in \mathcal{V}_P\) if the corresponding subgraphs in \(\mathcal{G}_h\) can form a valid pair (i.e., they satisfy the overlap constraints).
\end{itemize}

This construction can be done in polynomial time:
\begin{itemize}
    \item We can enumerate all subgraphs of size \(\alpha\) to \(\beta\) in \(O(|\mathcal{V}|^\beta)\) time.
    \item For each subgraph, we can check its density in \(O(|\mathcal{V}| + |\mathcal{E}|)\) time.
    \item We can check the overlap constraints for each pair of subgraphs in \(O(|\mathcal{V}|^2)\) time.
\end{itemize}

Now, finding \(k\) overlapping dense subgraphs in \(\mathcal{G}_h\) is equivalent to finding a \(k\)-clique in \(\mathcal{G}_P\). In particular, a 3-clique in \(\mathcal{G}_P\) corresponds to a solution for CTODS with \(k = 3\).

\textbf{4. Reduction from 3-Clique to CTODS:}\\
Given an instance of 3-Clique on graph \(\mathcal{G} = (\mathcal{V}, \mathcal{E}')\), we construct an instance of CTODS as follows:
\begin{itemize}
    \item Use the same graph \(\mathcal{G}\),
    \item Set \(k = 3\), \(\alpha = 3\), \(\beta = 3\), and \(r = 1\),
    \item Define the distance function \(d(\mathcal{G}[U], \mathcal{G}[Z])\) to return 0 if \(U = Z\) and 2 otherwise.
\end{itemize}

This construction ensures that:
\begin{itemize}
    \item We are looking for exactly three subgraphs (\(k = 3\)),
    \item Each subgraph must have exactly 3 vertices (\(\alpha = \beta = 3\)),
    \item Each subgraph must be fully connected (\(r = 1\) requires maximum density),
    \item The subgraphs must be identical (distance function).
\end{itemize}

A solution to this CTODS instance exists if and only if \(\mathcal{G}\) contains a 3-clique.

\textbf{5. Additional Cases:}\\
While the general problem is NP-complete, there are cases that can be solvable in polynomial time:
\begin{itemize}
    \item Case 1: when \(\alpha = \beta = 1\), we  simply select individual vertices; this step requires \(O(|\mathcal{V}|)\) time.
    \item When the hop size constraint is 1, we only consider direct neighbors; this step requires \(O(|\mathcal{V}| + |\mathcal{E}|)\) time.
    \item For specific density values that correspond to well-known structures (e.g., triangles in simple graphs), we might have polynomial-time algorithms.
    \item Case 2: when \(\alpha = \beta = 2\), the problem reduces to finding dense edges. This could lead to a polynomial-time solvable problem in bipartite graphs, as many graph problems become tractable on bipartite graphs.
    \item Case 3: when \(\alpha\) and especially \(\beta\) are close to or equal to \(|\mathcal{V}|\), the problem may become polynomial-time solvable. In this case, we are essentially looking for dense subgraphs that include most or all vertices of the original graph. This significantly reduces the search space and may allow for efficient algorithms.
\end{itemize}

These cases highlight that the hardness of the problem can vary significantly depending on specific constraints and graph structure. For instance, in bipartite graphs or when the subgraph size constraints approach the size of the entire graph, the problem can become polynomially-solvable.

However, it is important to note that these special cases do not contradict the NP-completeness of the general problem. The general case, where \(\alpha\) and \(\beta\) allow for a wide range of subgraph sizes and the graph structure is unrestricted, remains NP-complete.

By demonstrating that CTODS is in NP, providing polynomial-time reductions in both directions and considering both the trivial and these additional cases, we conclude that the general Constrained Top-k-Overlapping Densest Subgraphs problem is NP-complete while acknowledging that specific instances or constraints may lead to polynomial-time solvable variants.
\end{proof}

\begin{figure*}[t]
\includegraphics[width=0.95\textwidth, height=0.17\textheight]{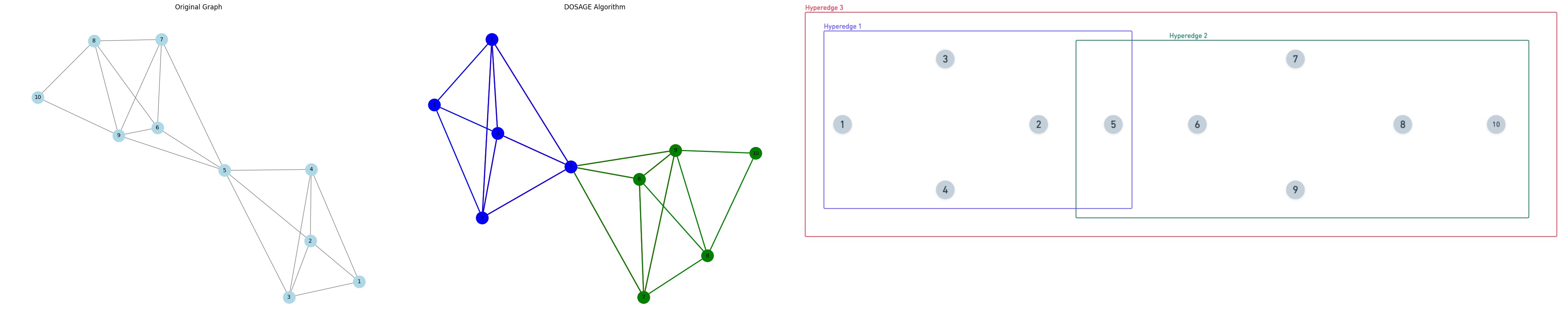}
\captionsetup{width=0.45\textwidth}
\caption{Conversion of a graph into hypergraph}
\label{Example}
\end{figure*}

\subsection{Algorithm}
As we understood the reason for setting constraints to the CTODS problem, we now present the pseudo-code of the DOSAGE algorithm. First, we discuss the DOSAGE algorithm, followed by the supporting functions as needed, and the hypegraph construction algorithm. Since we understood the reason for setting constraints to our algorithm, for a better understanding of the DOSAGE algorithm, we present the pseudo-code of this algorithm first, and then we will walk through each step of our algorithm. First, we start with our DOSAGE algorithm, which is the most critical aspect of our code. Since we have already talked about the density, distance, and objective function in the previous sections, we provided the code for them in the second algorithm as supporting functions.

The $DensestSubgraph$ function finds the subgraph that has the maximum density based on Goldberg's algorithm \cite{fast-goldberg}, constrained by minimum and maximum subgraph sizes ($\alpha$, $\beta$) and the diameter, $\delta$. The function repeatedly checks the density and diameter of the current subgraph. If size and diameter conditions are met, then the density of that subgraph is calculated. 

The output of the function is the densest subgraph found in our graph. The $IsDistinct$ function helps the $DensestDistinctSubgraph$ function to check whether a given subgraph $\mathcal{G}[\mathcal{S}]$ is distinct from all the subgraphs already stored in $W$. The $DensestDistinctSubgraph$ considers all possible subgraphs in the range of $\alpha$ to $\beta$ and checks if the diameter condition is met and is distinct from those already found. Finally, it calculates the objective function for that subgraph, and if the value is higher than the maximum, the subgraph is stored as the best.

\begin{algorithm}[H]
\caption{DOSAGE: Densest Overlapping Subgraphs via Adaptive Greedy Enumeration}
\begin{algorithmic}[1]

\Function{DensestSubgraph}{$\mathcal{G}, \alpha, \beta, \delta$}
    \State $\mathcal{G}_{\text{best}} \gets \text{null}$
    \State $d_{\text{best}} \gets 0$
    \State $\mathcal{G}_{\text{current}} \gets \mathcal{G}.\text{copy}()$
    \State $\text{degrees} \gets \text{ComputeDegrees}(\mathcal{G}_{\text{current}})$
    \While{$|\mathcal{V}(\mathcal{G}_{\text{current}})| > 0$}
        \If{$\text{Diameter}(\mathcal{G}_{\text{current}}) \leq \delta$}
            \State $d_{\text{current}} \gets \text{Density}(\mathcal{G}_{\text{current}})$
            \If{$d_{\text{current}} > d_{\text{best}}$ and $\alpha \leq |\mathcal{V}(\mathcal{G}_{\text{current}})| \leq \beta$}
                \State $d_{\text{max}} \gets d_{\text{current}}$
                \State $\mathcal{G}_{\text{best}} \gets \mathcal{G}_{\text{current}}.\text{copy}()$
            \EndIf
        \EndIf
        \State $\text{min\_degree} \gets \min(\text{degrees})$
        \State $\mathcal{V}_{\text{remove}} \gets \{v \in \mathcal{V}(\mathcal{G}_{\text{current}}) : \text{degree}(v) = \text{min\_degree}\}$
        \State $\text{RemoveNodes}(\mathcal{G}_{\text{current}}, \mathcal{V}_{\text{remove}})$
        \State $\text{UpdateDegrees}(\text{degrees}, \mathcal{G}_{\text{current}})$
        \If{$|\mathcal{V}(\mathcal{G}_{\text{current}})| < \alpha$}
            \State \textbf{break}
        \EndIf
    \EndWhile
    \State \Return $\mathcal{G}_{\text{best}}$
\EndFunction

\Function{IsDistinct}{$\mathcal{G}[\mathcal{S}], W$}
    \State \Return $\forall \mathcal{G}[\mathcal{S}_i] \in W : \text{Distance}(\mathcal{G}[\mathcal{S}], \mathcal{G}[\mathcal{S}_i]) > 0$
\EndFunction

\Function{DensestDistinctSubgraph}{$\mathcal{G}, W, \lambda, \alpha, \beta, \delta$}
    \State $d_{\text{max}} \gets 0$
    \State $\mathcal{G}_{\text{best}} \gets \text{null}$
    \For{$\text{subset\_size} \in [\alpha, \beta]$}
        \For{$\mathcal{S} \in \text{Combinations}(\mathcal{V}, \text{subset\_size})$}
            \State $\mathcal{G}[\mathcal{S}] \gets \text{InducedSubgraph}(\mathcal{G}, \mathcal{S})$
            \If{$\text{Diameter}(\mathcal{G}[\mathcal{S}]) > \delta$}
                \State \textbf{continue}
            \EndIf
            \If{$\text{IsDistinct}(\mathcal{G}[\mathcal{S}], W)$}
                \State $W_{\text{temp}} \gets W \cup \{\mathcal{G}[\mathcal{S}]\}$
                \State $d_{\text{current}} \gets \text{ObjectiveFunction}(W_{\text{temp}}, \lambda)$
                \If{$d_{\text{current}} > d_{\text{max}}$}
                    \State $d_{\text{max}} \gets d_{\text{current}}$
                    \State $\mathcal{G}_{\text{best}} \gets \mathcal{G}[\mathcal{S}]$
                \EndIf
            \EndIf
        \EndFor
    \EndFor
    \State \Return $\mathcal{G}_{\text{best}}$
\EndFunction

\vspace{1em}

\Function{DOSAGE}{$\mathcal{G}, k, \lambda, \alpha, \beta$}
    \If{$\mathcal{G}$ is connected}
        \State $\delta \gets 2 \cdot \text{AverageShortestPathLength}(\mathcal{G})$
    \Else
        \State $\delta \gets \log_2(|\mathcal{V}|)$
    \EndIf

    \State $\mathcal{G}_{\text{initial}} \gets \text{DensestSubgraph}(\mathcal{G}, \alpha, \beta, \delta)$
    \If{$|\mathcal{V}(\mathcal{G}_{\text{initial}})| > 0$}
        \State $W \gets \{\mathcal{G}_{\text{initial}}\}$
    \Else
        \State $W \gets \emptyset$
    \EndIf

    \While{$|W| < k$}
        \State $\mathcal{G}_{\text{next}} \gets \text{DensestDistinctSubgraph}(\mathcal{G}, W, \lambda, \alpha, \beta, \delta)$
        \If{$\mathcal{G}_{\text{next}} = \text{null}$ or $|\mathcal{V}(\mathcal{G}_{\text{next}})| = 0$}
            \State \textbf{break}
        \EndIf
        \State $W \gets W \cup \{\mathcal{G}_{\text{next}}\}$
    \EndWhile

    \State $\mathcal{G}_h^{(\alpha, \beta)} \gets \text{Hypergraph}(W)$ 
    \State \Return $\mathcal{G}_h^{(\alpha, \beta)}$ 
\EndFunction
\end{algorithmic}
\end{algorithm}

DOSAGE algorithm, which is the main function, finds the top-$K$ overlapping subgraphs in the graph by utilizing $DensestSubgraph$ and $DensestDistinctSubgraph$ functions. Then, it uses the Hypergraph function to create the hypergraph based on the subgraphs found by the algorithm. In the hypergraph function, the vertex set of the hypergraph $\mathcal{V}$ is derived directly from the vertices of the original graph $\mathcal{V}(\mathcal{G})$, which means that all vertices from the original graph are covered by the hypergraph. $\mathbf{X}_i$ represents the initial feature matrix for the nodes in the hypergraph. The resulting hypergraph $\mathcal{R}$ and the node feature matrix $\mathbf{X}_i$ are passed as inputs to a Hypergraph Neural Network (HGNN). The function returns the output feature matrix $\mathbf{R}_{\text{out}}$, which represents the transformed or learned features for the nodes after passing them through the HGNN.

\begin{algorithm}[t!]
\caption{Supporting Functions for DOSAGE}
\begin{algorithmic}[1]
\Function{Density}{$\mathcal{G}$}
    \If{$|\mathcal{V}(\mathcal{G})| = 0$}
        \State \Return \text{null}
    \EndIf
    \State \Return $\frac{|\mathcal{E}(\mathcal{G})|}{|\mathcal{V}(\mathcal{G})|}$
\EndFunction

\Function{Distance}{$\mathcal{G}[U], \mathcal{G}[Z]$}
    \If{$|\mathcal{V}(U)| = 0$ or $|\mathcal{V}(Z)| = 0$}
        \State \Return 2
    \EndIf
    \If{$\mathcal{V}(U) = \mathcal{V}(Z)$}
        \State \Return 0
    \EndIf
    \State $\text{intersection\_size} \gets |\mathcal{V}(U) \cap \mathcal{V}(Z)|$
    \State \Return $2 - \frac{\text{intersection\_size}^2}{|\mathcal{V}(U)| \cdot |\mathcal{V}(Z)|}$
\EndFunction

\Function{ObjectiveFunction}{$W, \lambda$}
    \State $\text{total\_density} \gets \sum_{\mathcal{G}[\mathcal{S}] \in W} \frac{|\mathcal{E}(\mathcal{G}[\mathcal{S}])|}{|\mathcal{V}(\mathcal{G}[\mathcal{S}])|}$
    \State $\text{total\_distance} \gets \sum_{i < j < |W|} \text{Distance}(\mathcal{G}[S_i], \mathcal{G}[S_j])$
    \State \Return $\text{total\_density} + \lambda \cdot \text{total\_distance}$
\EndFunction

\end{algorithmic}
\end{algorithm}

\begin{algorithm}[t!]
\caption{Hypergraph Pass to HGNN}
\begin{algorithmic}[1]
\Function{Hypergraph}{$W$}
    \State $\mathcal{G}_h \gets (\mathcal{V}, \mathcal{E}_h)$ 
    \State $\mathcal{V} \gets \mathcal{V}(\mathcal{G})$ 

    \ForAll{$\mathcal{G}[\mathcal{S}] \in W$}
        \State $\mathcal{E}_h \gets \mathcal{E}_h \cup \{\mathcal{V}(\mathcal{G}[\mathcal{S}])\}$ 
    \EndFor

    \State $\mathbf{X}_i \gets \text{InitializeNodeFeatures}(\mathcal{V})$ 
    \State $\mathbf{R}_{\text{out}} \gets \text{HGNN}(\mathcal{G}_h, \mathbf{X}_i)$ 

    \State \Return $\mathbf{R}_{\text{out}}$ 
\EndFunction
\end{algorithmic}
\end{algorithm}

By using the DOSAGE algorithm, the methodology ensures that the resulting hypergraph captures the most significant dense regions, accounting for potential overlaps and ensuring high-quality hyperedges. This approach provides a scalable and efficient means to enhance hypergraph-based representations and analyses. As such, we establish an efficient method for identifying interconnected substructures within the graph. The resulting hypergraph \(\mathcal{G}_h^{(\alpha, \beta)}\) can capture high-order correlations in the data, providing a more expressive and informative representation than traditional graph structures \cite{node-classification}. This is particularly useful in applications such as community detection in social networks and motif discovery in biological networks, where overlapping dense regions are of interest \cite{dense}. 
% The features that learned through passing our hypergraph to the HGNN could be used for machine learning tasks such as node classification, link prediction, or other graph-related tasks \cite{node-classification}.

\section{Experiments}
In this section, we compare the performance of our hypergraph and other methods on node classification tasks. We experimented with DOSAGE on two datasets, the Cora and Cooking-200 datasets \cite{cora}. 

% \subsection{Datasets}
% The Cooking dataset has been collected from Yummly.com \footnote{\href{https://www.yummly.com}{https://www.yummly.com}}, in which vertex denotes the dish and hyperedge denotes the ingredient. Each dish is also associated with category information, which indicates the dish’s cuisine, such as Chinese, Japanese, French, and Russian. The Cora dataset is a citation network of 2,708 machine-learning papers organized into seven distinct classes. These papers are interlinked by 5,429 citations, forming a directed graph that maps out how papers cite each other. Each paper is represented by a binary word vector derived from a dictionary of 1,433 unique words, indicating the presence or absence of specific words in the paper.

\subsection{Results and Discussion}
The experimental results for the Cora dataset are presented in Table I, while Table II shows the results for the Cooking-200 dataset. We compare the accuracy and F1-score of different models, including Graph Convolutional Network (GCN) \cite{gcn}, Graph Attention Network (GAT) \cite{gat}, Graph Sample and Aggregation (GraphSAGE) \cite{sage}, Graph Isomorphism Network (GIN) \cite{gin}, Graph Convolution (GraphConv) \cite{graph-conv}, Hypergraph Convolutional Network (HyperGCN) \cite{hgcn}, Hypergraph Attention Network (Hyper-Atten) \cite{hyperatt}, Hypergraph Neural Network (HGNN) \cite{HGNNs}, Hypergraph Neural Network\textsuperscript{+} (HGNN\textsuperscript{+}) \cite{HGNNs}, and our proposed DOSAGE method.

% Table for Cora Dataset
\begin{table}[t]
\centering
\caption{Experimental results on the Cora Dataset.}
\renewcommand{\arraystretch}{1.3}
\begin{tabular}{lcc}
\hline
\textbf{} & \textbf{Accuracy} & \textbf{F1\_score} \\
\hline
\textbf{GCN} & 0.5411 & 0.5180 \\
\textbf{GAT} & 0.5519 & 0.5237 \\
\textbf{GraphSAGE} & 0.5741 & 0.5331 \\
\textbf{GIN} & 0.5767 & 0.5590 \\
\textbf{GraphConv} & 0.5743 & 0.5655 \\
\textbf{HyperGCN} & 0.5844 & 0.5701 \\
\textbf{Hyper-Atten} & 0.6589 & 0.6312 \\
\textbf{HGNN} & 0.6650 & 0.6478 \\
\textbf{HGNN\textsuperscript{+}} & 0.6701 & 0.6512 \\
\textbf{HGNN\textsuperscript{+} using DOSAGE} & \textbf{0.7103} & \textbf{0.7067} \\
\hline
\end{tabular}
\end{table}
Table I shows that the $DOSAGE$ algorithm significantly outperforms traditional GNN models, including GCN, GAT, and GraphSAGE, as well as hypergraph-based models, namely HyperGCN, HGNN, and Hyper-Atten. The proposed $DOSAGE$ model achieves the highest accuracy of 71.03\% and an F1-score of 70.67\%. These results demonstrate that $DOSAGE$, by effectively generating hyperedges using the densest overlapping subgraphs, captures complex relationships between nodes that are not adequately represented by traditional GNNs or even by other hypergraph models.

In the Cooking-200 dataset (Table II), $DOSAGE$ again outperforms all other models, achieving an accuracy of 45.72\% and an F1-score of 40.19\%. While the gains in performance are not as large as those seen in the Cora dataset, $DOSAGE$ still shows a clear advantage over HGNN and HGNN\textsuperscript{+}. The challenges posed by the Cooking-200 dataset, such as a higher degree of sparsity and more complex relationships between dishes and ingredients, are better addressed by the $DOSAGE$ algorithm's ability to model these intricacies through the densest overlapping subgraphs.

The results from both datasets indicate that the $DOSAGE$ algorithm's hypergraph construction method, based on the densest overlapping subgraphs, offers superior performance in node classification tasks compared to existing GNN and HGNN models. By considering not only the density but also the overlapping nature of subgraphs, $DOSAGE$ effectively captures richer and more nuanced relationships within the data. This results in better node embeddings, leading to improved classification performance. 
% Moreover, the consistent performance improvement across different datasets highlights the generalizability and robustness of the $DOSAGE$ algorithm.

However, it is important to note that the $DOSAGE$ model took approximately three minutes longer to execute than HGNN\textsuperscript{+}. While this increase in computation time is relatively minor and does not detract from the overall performance gains, it does suggest a potential area for improvement. 
% Optimizing the efficiency of the $DOSAGE$ algorithm could help reduce this time difference, making the model even more competitive in practical applications.

\section{Future Work}
%While the $DOSAGE$ algorithm has demonstrated significant improvements in the construction of hypergraphs for node classification tasks, there are several avenues for future work that could further enhance its efficiency and effectiveness.

%\subsection{Improving Computational Efficiency}
One of the key areas for future development is optimizing the computational efficiency of the $DOSAGE$ algorithm. Currently, the process of identifying and generating the densest overlapping subgraphs can be computationally expensive, particularly for large-scale datasets with high node and edge counts. Notably, the $DOSAGE$ model took approximately three minutes longer to execute than HGNN\textsuperscript{+}. While this increase in time is not substantial, it highlights the need for further optimization. 
% To address this, we plan to explore advanced optimization techniques that reduce the time complexity of the algorithm. This could include parallel processing strategies, more efficient data structures, and heuristic-based approaches that approximate dense subgraphs with lower computational overhead. The goal is to maintain the accuracy benefits of $DOSAGE$ while making it scalable to larger and more complex graphs.

% Table for Cooking-200 Dataset
\begin{table}[t]
\centering
\caption{Experimental results on the Cooking-200 Dataset.}
\renewcommand{\arraystretch}{1.3}
\begin{tabular}{lcc}
\hline
\textbf{} & \textbf{Accuracy} & \textbf{F1\_score} \\
\hline
\textbf{GCN} & 0.3110 & 0.2680 \\
\textbf{HGNN} & 0.3220 & 0.2749 \\
\textbf{HGNN\textsuperscript{+}} & 0.4294 & 0.3725 \\
\textbf{HGNN\textsuperscript{+} using DOSAGE} & \textbf{0.4572} & \textbf{0.4019} \\
\hline
\end{tabular}
\end{table}

%\subsection{Dynamic Hyperedge Construction}
Another promising direction for future work is developing a dynamic mechanism for hyperedge construction. In the current implementation of $DOSAGE$, the hyperedges formed from subgraphs are static, meaning they do not change in the HGNN training phase once they are created. As such, we envision a more adaptive system where the subgraphs can update themselves based on feedback from the hypergraph. This feedback loop would allow the subgraphs to refine their structure over time, potentially leading to even more accurate and representative hyperedges.

% For instance, after initial hypergraph construction and a round of training with the Hypergraph Neural Network (HGNN), performance metrics and learned features could be analyzed to identify areas where hyperedges might be improved. Subgraphs could then be dynamically adjusted—such as by adding or removing nodes or by re-evaluating the density and overlap criteria—to better capture the underlying relationships in the data. This adaptive approach would create a more iterative and responsive hypergraph modeling process, potentially leading to continuous improvements in model performance as the hyperedges evolve.

\subsection{Conclusion}
In this paper, we introduced the DOSAGE algorithm, a novel approach to hypergraph construction that leverages densest overlapping subgraphs to improve node classification tasks. Unlike traditional graph neural networks (GNNs) and existing hypergraph neural networks (HGNNs), DOSAGE focuses on capturing complex relationships within data by constructing hyperedges that reflect more intricate and overlapping structures.

The key contributions of our work include the development of the DOSAGE algorithm, which provides a robust method for generating hyperedges that enhance the expressiveness of hypergraph models. We demonstrated that this approach not only addresses the limitations of existing hyperedge construction techniques but also significantly improves classification accuracy across different datasets.

The numerical results presented in this paper underscore the advantages of the DOSAGE algorithm. On the Cora dataset, DOSAGE achieved the highest accuracy of 71.03\% and an F1-score of 70.67\%, outperforming several baseline models, including GCN, GAT, GraphSAGE, and other hypergraph-based methods like HyperGCN and Hyper-Atten. Similarly, in the Cooking-200 dataset, DOSAGE continued to show superior performance with an accuracy of 45.72\% and an F1-score of 40.19\%, demonstrating its effectiveness even in more challenging, sparsely connected datasets.

Overall, DOSAGE offers a powerful tool for hypergraph-based learning, providing a more nuanced understanding of data relationships and leading to improved outcomes in node classification tasks. The contributions and results presented in this paper pave the way for future research and applications in hypergraph neural networks, with the potential to extend these methods to even more complex and large-scale problems.

\section{Acknowledgments}
This research work has been partially supported by the Natural Sciences and Engineering Research Council of Canada, NSERC, Vector Institute for Artificial Intelligence. This work has been made possible by using the facilities of the Shared Hierarchical Academic Research Computing Network (SHARCNET: www.sharcnet.ca) and Compute/Calcul Canada.

\vspace{12pt}
\bibliographystyle{IEEEtran}
\bibliography{bibliography}
\end{document}